\documentclass[letterpaper, 10 pt, conference]{ieeeconf}  

\IEEEoverridecommandlockouts                              
\usepackage{cite}
\usepackage{amsmath,amssymb,amsfonts,dsfont}
\usepackage{algorithmic}
\usepackage{graphicx}
\usepackage{textcomp}
\usepackage{setspace}
\usepackage{booktabs}
\usepackage{lipsum}

\usepackage{epsfig}
\usepackage{times} 
\usepackage{svg}
\usepackage[linesnumbered,ruled,vlined]{algorithm2e} 

\usepackage[nolist, nohyperlinks]{acronym}

\usepackage{bm}
\usepackage{breqn}
\usepackage{amsmath}
\usepackage{mathtools}
\usepackage{amssymb}
\usepackage{tabularx}

\usepackage{amsthm}

\newcommand{\reals}{\mathbb{R}}
\newcommand{\R}{\reals}
\newcommand{\Rnonneg}{\reals_{\geq 0}}

\newcommand{\naturals}{\mathbb{N}}
\newcommand{\N}{\naturals}
\newcommand{\Znonneg}{\mathbb{Z}_{\geq 0}}

\newcommand{\Pbb}{\mathbb{P}}
\newcommand{\Qbb}{\mathbb{Q}}

\newcommand{\Acal}{\mathcal{A}}
\newcommand{\Bcal}{\mathcal{B}}
\newcommand{\Ccal}{\mathcal{C}}
\newcommand{\Dcal}{\mathcal{D}}

\newcommand{\Ical}{\mathcal{I}}

\newcommand{\Ncal}{\mathcal{N}}
\newcommand{\Ocal}{\mathcal{O}}

\newcommand{\Scal}{\mathcal{S}}

\newcommand{\Ucal}{\mathcal{U}}

\newcommand{\Wcal}{\mathcal{W}}
\newcommand{\Xcal}{\mathcal{X}}

\newcommand{\eqn}[1]{\begin{align} #1 \end{align}}
\newcommand{\eqnN}[1]{\begin{align*} #1 \end{align*}}

\theoremstyle{plain}
\newtheorem{theorem}{Theorem}

\newtheorem{lemma}{Lemma}
\newtheorem{problem}{Problem}
\newtheorem{definition}{Definition}

\theoremstyle{definition}
\newtheorem{assumption}{Assumption}
\newtheorem{remark}{Remark}

\theoremstyle{remark}

\makeatletter
\let\NAT@parse\undefined
\makeatother
\usepackage{hyperref}
\hypersetup{
colorlinks, 
    linkcolor={red!50!black},
    citecolor={blue!50!black},
    urlcolor={blue!80!black}
}
\usepackage{cleveref}
\usepackage[table]{xcolor}
\usepackage[dvipsnames]{xcolor}
\crefformat{problem}{Problem~#2#1#3}
\crefformat{assumption}{Assumption~#2#1#3}

\usepackage[nolist,nohyperlinks]{acronym}
\newcommand{\gk}{\texttt{gatekeeper}}
\newcommand{\sgk}{\texttt{DRS-gatekeeper}}

\newcommand{\w}{{\boldsymbol w}}

\acrodef{PCIS}[$\alpha$-PCIS]{$\alpha$-probabilistic controlled invariant set}
\acrodef{PCI}[$\alpha$-PCI]{$\alpha$-probabilistic controlled invariant}
\acrodef{DRCC}[DRCC]{distributionally robust chance constraint}

\title{\LARGE \bf Distributionally Robust Safety Under Arbitrary Uncertainties: \\ A Safety Filtering Approach
}

\crefname{definition}{Def.}{Defs.}
\author{Daniel M. Cherenson*, Haejoon Lee*, Taekyung Kim, and Dimitra Panagou
\thanks{*These authors contributed equally to this work.}
\thanks{This research was supported by the Center for Autonomous Air Mobility and Sensing (CAAMS), an NSF IUCRC, under Award Number 2137195, an NSF CAREER under Award Number 1942907, and the Air Force Office of Scientific Research (AFOSR) under Award No. FA9550-23-1-0163.}
\thanks{All authors are with the Department of Robotics,
        University of Michigan, Ann Arbor, MI, USA
        {\tt\small \{dmrc, haejoonl, taekyung, dpanagou\}@umich.edu}}%
\thanks{$^a$Project Page:
\href{https://drs-gk.github.io/}{https://drs-gk.github.io/}}}

\begin{document}

\maketitle
\thispagestyle{empty}
\pagestyle{empty}

\begin{abstract}
We study how to ensure probabilistic safety for nonlinear systems under distributional ambiguity. Our approach builds on a backup-based safety filtering framework that switches between a high-performance nominal policy and a certified backup policy to ensure safety. To handle arbitrary uncertainties from ambiguous distributions, i.e., where the distribution is not of specific structure and the true distribution is unknown, we adopt a distributionally robust (DR) formulation using Wasserstein ambiguity sets. Rather than solving a high-dimensional DR trajectory optimization problem online, we exploit the structure of backup-based safety filtering to reduce safety certification to a one-dimensional search over the switching time between nominal and backup policies. We then develop a sampling-based certification procedure with finite-sample guarantees, where empirical failure probabilities are compared against a Wasserstein-inflated threshold. We validate our method across three systems, from a Dubins vehicle to a high-speed racing car and a fighter jet, demonstrating the broad applicability and computational efficiency. \href{https://drs-gk.github.io/}{[Project Page]}$^a$
\end{abstract}

\section{Introduction}


Ensuring the safety of autonomous robotic systems remains a challenge in the presence of complex nonlinear dynamics with actuation limits, uncertainty due to model mismatch, perception errors, and exogenous disturbances. These uncertainties might be unknown or poorly characterized, complicating the synthesis of a safe controller that does not compromise mission performance.

Safety filters have enabled modular architectures, allowing the use of high-performance, safety-agnostic nominal policies and intervening only when a safety violation is imminent~\cite{garg2024advances, hsu2024safety, bansal2017hamilton, ames2019control}. In the presence of disturbances, robust safety filters~\cite{cosner2023robust, knoedler2025safety} provide guarantees under worst-case bounds, while stochastic filters~\cite{wang2025safe,singletary2023safe} assume known disturbance distributions. However, both approaches rely on accurate or restrictive uncertainty models, and thus may fail when the true distribution is unknown, partially observed, or shifted.

Distributional shift or ambiguity is addressed by distributionally robust (DR) safety-critical planning and control~\cite{wiesemann2014distributionally, safaoui2024distributionally}. These approaches account for worst-case uncertainty by enforcing safety against an ambiguity set, that is, a family of all probability distributions consistent with available data, defined by moments or Wasserstein distances~\cite{hakobyan2023distributionally, delage2010distributionally}. This framework has been applied to handle uncertain disturbances in perception~\cite{long2026sensor, ham2026dro,sung2025addressing} and system dynamics~\cite{zhang2024distributionally, aolaritei2023wasserstein}. 

Despite their benefits, integrating DR techniques directly into safety filters remains challenging. Incorporating and verifying DR safety constraints for general nonlinear systems is often computationally expensive~\cite{rahimian2019distributionally}, requiring optimization over a worst-case distribution within high-dimensional ambiguity sets. For instance, nonlinear DR-MPC frameworks can require hundreds of seconds to solve per time step~\cite{hakobyan2021wasserstein}, rendering them unsuitable for real-time safety filtering.

\begin{figure}
    \centering
    \includegraphics[width=0.98\linewidth]{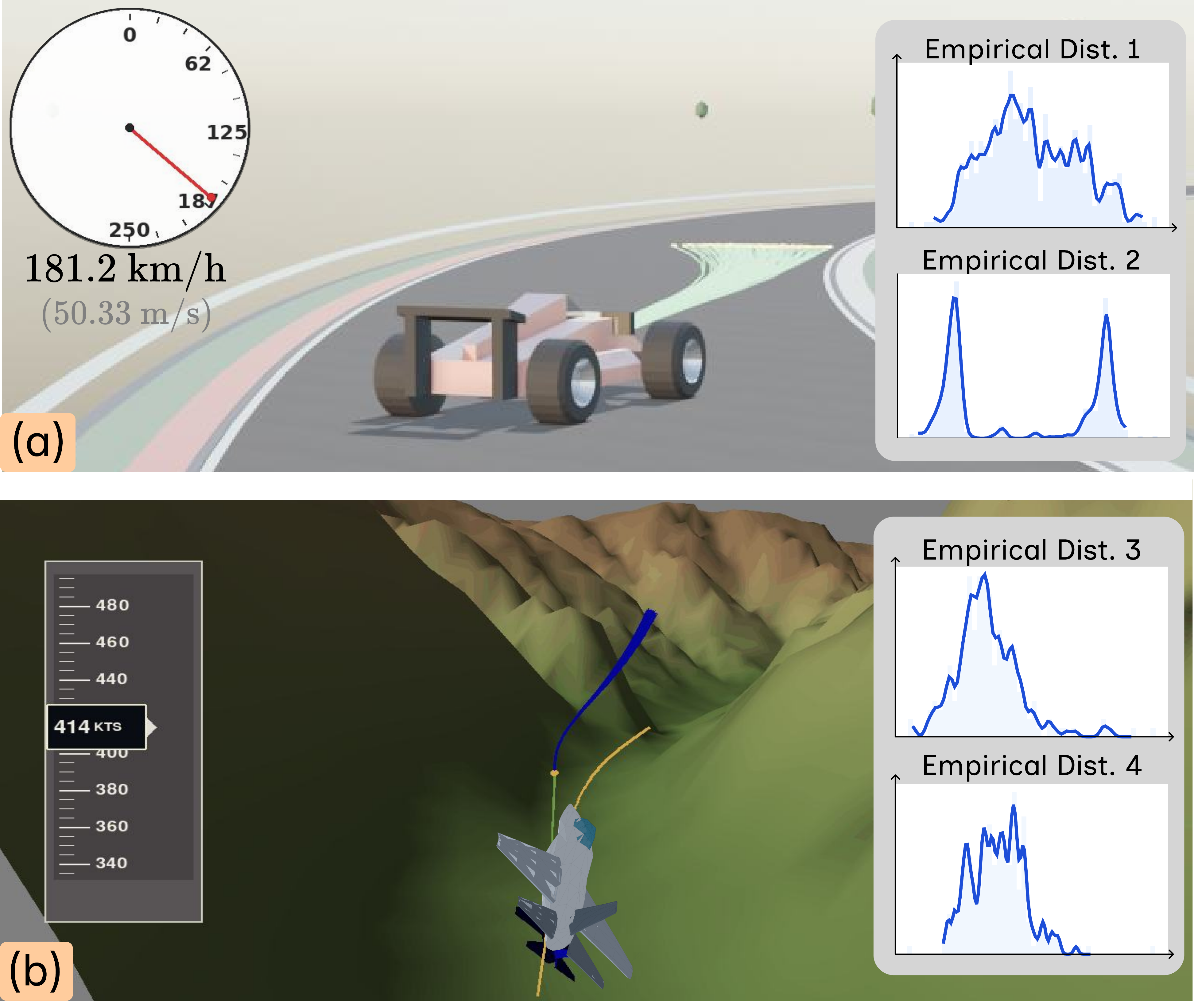}
\caption{Visualizations of (a) a 7-dimensional Formula 1 racecar with 3 control inputs and 4 uncertainty channels, and (b) a 12-dimensional F-16 fighter jet with 4 control inputs and 6 uncertainty channels, operating under complex, non-Gaussian empirical uncertainty distributions.}
    \label{fig:hero_fig}
    \vspace{-20pt}
\end{figure}

Several studies~\cite{summers2018distributionally, aolaritei2023wasserstein, ren2023chance,schuurmans2023safe} reformulate DR problems into deterministic, computationally tractable optimization problems. Nevertheless, these reformulations frequently require restrictive assumptions on the dynamics~\cite{aolaritei2023wasserstein,summers2018distributionally} and uncertainty structure (e.g., additive noises)~\cite{aolaritei2023wasserstein, schuurmans2023safe}. More recently,~\cite{long2026sensor} integrates DR methods with Control Barrier Functions (CBFs) to construct safety filters under perception uncertainty. Still, this approach is restricted to systems that are affine in both control input and noise.

Our approach builds on backup-based safety filters~\cite{bastani2021safe, chen2021backup, agrawal2024gatekeeper}, which ensure safety using pre-verified backup policies and intervene only when necessary. This architecture is well suited to high-dimensional nonlinear systems with input constraints, avoiding manual construction and feasibility issues of CBFs. However, existing methods only assume deterministic dynamics~\cite{bastani2021safe, chen2021backup} or restrictive uncertainty models, such as uniformly bounded disturbances~\cite{knoedler2025safety, bastani2021stat, agrawal2024gatekeeper}, limiting their applicability under unknown or misspecified distributions.


We propose \textbf{Distributionally Robust Stochastic \gk{} (\sgk{})}, a real-time, backup-based safety filter that provides probabilistic safety guarantees under arbitrary uncertainty structure (non-additive, non-affine). We reformulate the infinite-dimensional DR safety problem into a one-dimensional search for the switching time between nominal and backup policies, enabling efficient real-time implementation. Safety is certified using samples drawn only from the nominal distribution, while robustness to distributional shift is achieved by a Wasserstein-derived inflated threshold for checking unsafe rollouts.
The main contributions are:


\begin{itemize}
    \item We introduce \sgk{}, a safety filter for nonlinear systems with arbitrary disturbance structure and distributional ambiguity.

  {\item We establish a theoretical connection between the failure probability under a nominal distribution and the worst-case failure probability over a Wasserstein ambiguity set to derive a tractable finite-sample probabilistic safety certificate under distributional ambiguity.
  }
    \item We validate our method across a Dubins vehicle, a Formula 1 racing car, and an F-16 fighter jet, demonstrating better safety-performance trade-offs compared to baseline methods and broad applicability.
\end{itemize}

\section{Problem Formulation}

Let $\Znonneg$ and $\mathbb{R}$ denote the sets of nonnegative integers and real numbers. Let $\Dcal(\Wcal)$ be the space of all probability distributions supported on set $\Wcal$ and let $\otimes$ denote the product measure. $\Gamma(\mathbb{Q}, \mathbb{Q}')$ represents the space of all joint distributions with marginals $\mathbb{Q}$ and $\mathbb{Q}'$. The cumulative binomial distribution function, $\text{Bin}(k; N, q)$, yields the probability of $k$ or fewer successes in $N$ trials with probability $q$. Let $\mathds{1}(\cdot)$ denote the indicator function.

\subsubsection{System Dynamics and Probability Distributions}
 Consider a discrete-time, nonlinear system
 \eqn{x_{t+1} & = f(x_t, u_t, w^f_t) \label{eq:system}} 
where $x_t \in \mathcal{X} \subseteq \mathbb{R}^n$, $u_t \in \mathcal{U} \subset \mathbb{R}^m$, and $w^f_t \in \mathcal{W}^f \subset \mathbb{R}^p$ denote the state, control input, and process noise at time $t \in \Znonneg$, respectively. The dynamics function $f : \Xcal \times \Ucal \times \Wcal^f \to \Xcal$ is Lipschitz and $w_t^f$ is drawn from a state- and input-dependent distribution $\Pbb^f(\cdot | x_t, u_t) \in \Dcal(\Wcal^f)$.

Under a feedback controller $\pi : \Xcal \to \Ucal$, the closed loop system is given by
\eqn{\label{eq:closed_loop}
    x_{t+1} &= f(x_t, \pi(x_{t}), w^f_t).
}For a planning horizon $T$, we denote the noise {trajectory as $\w^f_T \coloneqq (w^f_{t}, \dots, w^f_{t+T-1}) \sim \Pbb^\pi_{{T}}$,} where 
\eqn{{\Pbb}_{{T}}^\pi(\w^f_T) &= \bigotimes_{\tau=t}^{t+T-1} {\Pbb^f}(w^f_\tau \mid x_\tau, \pi(x_\tau))~\text{ s.t. }~\eqref{eq:closed_loop}.\label{eq:process_noise}}

\subsubsection{Safety Constraint}

We require the system~\eqref{eq:system} to remain safe throughout its operation. Safety is defined as the state $x_t$ remains outside the unsafe set $\Ocal(\theta)$ for all time $t$, where we assume $\Ocal(\theta)$ is closed and parameterized by an unknown, static parameter vector $\theta \in \Theta \subset \R^{v}${, which abstracts safe-set parameters, e.g., obstacle location, size, or parameters for distance.} We define $h : \Xcal \times \Theta \to \R$ as a distance function to the unsafe set, with safe set as $\Scal(\theta):=\{x_t \in \Xcal \mid h(x_t,\theta)\ge 0\}$. We assume that $h$ is globally Lipschitz w.r.t. $\theta$.

The system senses $\Ocal(\theta)$ through a perception module that processes observations $z_t = g(x_t, \theta, w^z_t)$
with noise $w^z_t \sim \Pbb^z$. When $g$ is non-invertible, recovery of $\theta$ becomes intractable. To address this, we abstract the measurement process into a conditional distribution $\Pbb^\theta_t = \Pbb(\cdot \mid x_t, z_t) \in \Dcal(\Theta)$, representing the agent’s belief of the parameters of the unsafe set. Through this distribution, we generate parameter hypotheses $w^\theta_t \sim \Pbb^\theta_t$ through sampling.

The process noise $\w^f_{{T}}$ and perceived unsafe-set parameter $w^\theta_t$ are assumed to be independent. For notational simplicity, we define the joint distribution of these noises as the product measure $\Pbb_{{T}} = \Pbb^\pi_T \otimes \Pbb^\theta_t$,
and we lump all the noises into a single random variable $\w_{{T}} \sim \Pbb_{{T}}, \ \w_{{T}}\in \Wcal \subset \R^{v+Tp}$.

With the uncertainty $\w_T$, we aim to provide probabilistic safety. Let $\varphi_\tau^\pi(x_t, \w^f_{{T}})$ be the solution of the closed-loop dynamics~\eqref{eq:closed_loop} under policy $\pi$ {propagated for $\tau \in [0,T]$ steps}, with initial condition $x_t$ and noise trajectory $\w_{{T}}^f$. {Since our objective is to ensure safety over the entire predicted evolution of the system, we define probabilistic safety over the infinite horizon, corresponding to the limit $T\to\infty$, as
\eqn{\Pr{}_{\Pbb_{\infty}}\bigg[\bigcap_{\tau = 0}^\infty h(\varphi_\tau^\pi(x_t, \w^f_{\infty}), w^\theta_{t}) \ge 0\bigg] \ge 1-\varepsilon, \tag{CC}\label{eq:cc}}}where $\varepsilon \in (0,1)$ is the allowable probability of failure. {However, the infinite-horizon safety constraint in~\ref{eq:cc} is intractable to verify. Using the notions of controlled invariant set and backup policy, defined in the next section, we verify the infinite-horizon chance constraint satisfaction with a chance constraint defined over a finite horizon $T$}.

\subsubsection{Controlled Invariant Set and Backup Policy}

We consider a \textit{nominal policy} $\pi_N : \Xcal \to \Ucal$ that is designed for mission performance, but does not satisfy the safety constraints. Our goal is to track the nominal policy for as long as possible while ensuring safety. To this end, we switch to a backup policy that renders the safe set forward invariant with high probability only when the nominal policy would lead to unsafe behavior. To define the backup policy, first we define a \acl{PCIS}~\cite{gao2021computing}:
\begin{definition}[Infinite-Horizon \ac{PCIS}] Let $\Ccal = \{x \in \Xcal ~|~ h_c(x) \ge 0\}$ be a set, where $h_c : \Xcal \to \mathbb{R}$. For the system~\eqref{eq:system}, a controller $\pi : \Xcal \to \Ucal$ renders $\Ccal$ \textbf{\ac{PCI}} if for $\alpha \in [0,1]$, the closed-loop system~\eqref{eq:closed_loop} satisfies
\eqn{\label{eq:pci_alpha}{
    x_t \in \Ccal \implies \Pr{}_{\Pbb_\infty}[x_
\tau\in \Ccal, \; \forall 
\tau\ge t] \ge 1-\alpha, ~ \forall t \in \Znonneg.}}
\end{definition}

Methods for computing these sets have been developed in the literature~\cite{pola2006invariance, abate2008probabilistic, gao2021computing}. Next, we define a backup policy that reaches a set $\Ccal$ and renders it \ac{PCI}.
\begin{definition}[Backup Policy]
\label{def:backup}
    A policy $\pi_B : \Xcal \to \Ucal$ is a \textbf{backup policy} for {$\Ccal\subset \Scal(\theta)\subset \Xcal$ defined for all $t \ge t_0$ and $\theta\in \Theta$} if, for the closed-loop system~\eqref{eq:closed_loop}, there exists a neighborhood ${\Ncal(T)} \subset \Xcal$ of $\Ccal$ such that $\Ccal$ is reachable {within $T$, i.e.,}
    \eqn{
        x_\tau \in {\Ncal(T)} \implies x_{\tau + {T}} \in \Ccal,~ \forall \tau \ge t_0,
    }
    and $\pi_B$ renders $\Ccal$ \ac{PCI}, i.e.,
    \eqn{
        x_{\tau + {T}} \in \Ccal \implies \Pr{}_{\Pbb_\infty}[x_{t'} \in \Ccal, \; \forall t' \ge \tau + {T}] \ge 1-\alpha.
    }
\end{definition}

\begin{assumption}
    We have a backup policy $\pi_B$ that renders a known set {$\Ccal\subset\Scal(\theta)$ \ac{PCI} for all $\theta\in \Theta$, with $\alpha < \varepsilon$.}
    \label{assum:exist_backup}
\end{assumption}

Constructive methods for designing such backup policies are studied in~\cite{pola2006invariance, abate2008probabilistic}, which provide conditions and algorithms for computing such policies.

\begin{figure*}[t]
    \centering
    \vspace{6pt}
    \includegraphics[width=\textwidth]{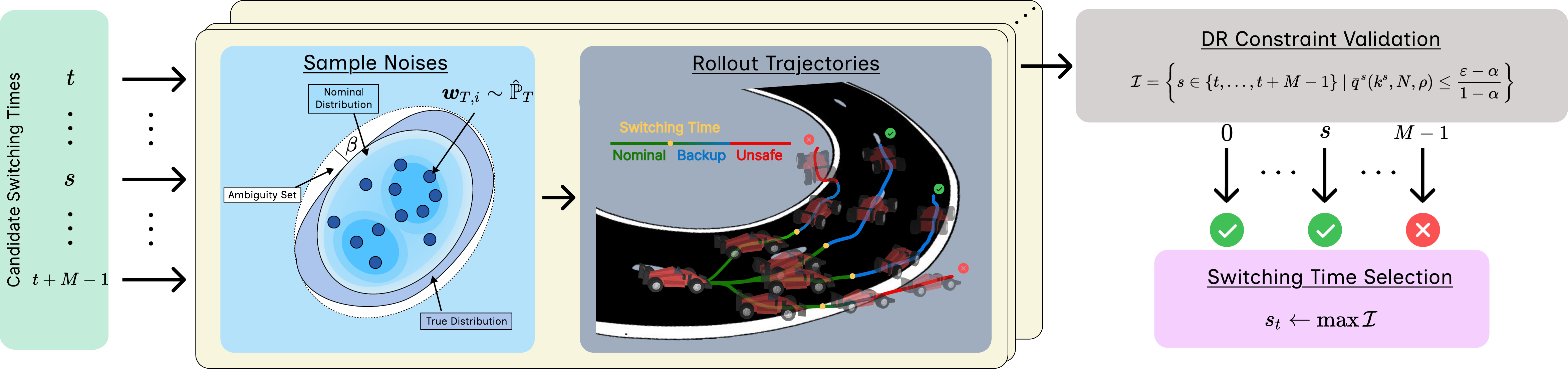}
    \caption{Visual flow chart of~\sgk{}. At each time step, for every candidate switching time $s\in\{t,\dots, t+M-1\}$, we sample $N$ noise trajectories from the nominal noise distribution and perform rollouts to evaluate safety. It then counts constraint violations and computes a distributionally robust upper bound on the failure probability. Finally, we select the largest switching time satisfying the DR failure probability bound, or uses the previous switching time if no candidate is certified. }
    \label{fig:flowchart}
    \vspace{-20pt}
\end{figure*}

We want to track the nominal policy as long as possible. Thus, with~\Cref{assum:exist_backup}, our problem can be cast as determining \textit{when to switch} from $\pi_N$ to $\pi_B$. Formally, 
\begin{definition}[Switching Policy]
    A \textbf{switching policy} $\pi_{S} : \Xcal \times \Znonneg \times \Znonneg \to \Ucal$ is a policy that switches from the nominal to the backup policy
    \eqn{
    \label{eq:switched_time}
        {{\pi_S(x, t; s)} = 
        \begin{cases}
            \pi_N(x), & s > t\\
            \pi_B(x), & s \le t,
        \end{cases} }}{where the shorthand $\pi_S(;s)$ is the switching policy with switching time $s$.}
\end{definition}
{Note that $\Ccal\subset\Ncal(T)$, but $\Ncal(T)$ may not be safe.} Thus, a trajectory under $\pi_B$ is not necessarily safe, motivating the search for a switching time $s_t^\star$ such that the entire closed-loop trajectory under $\pi_S(;s_t^\star)$ is safe. We define a function that evaluates constraint satisfaction over the horizon $T$ and determines whether the trajectory reaches an \ac{PCIS}.
\begin{definition}[Finite-Horizon Safety Function]
    {Given an \ac{PCIS} $\Ccal$ and a switching policy $\pi_S$}, the induced \textbf{finite-horizon safety function} $H : \Xcal \times \Znonneg \times  \Znonneg \times \Wcal \to \R$ with shorthand $H^{x,s}_{t}(\cdot) = H(x, t,s, \cdot)$ is defined as \eqnN{H(x, t,s, \w_T) = \min\bigg\{ & \min_{0 \le\tau \le T} h(\varphi_\tau^{\pi_S(;s)}(x, \w^f_T), w^\theta_{t}), \\&h_c(\varphi_T^{\pi_S(;s)}(x, \w^f_T)) \bigg\}.}
\end{definition}
{

Under~\Cref{assum:exist_backup}, once the system reaches $\Ccal$, the backup policy $\pi_B$ keeps it in $\Ccal$ for all $\tau \geq t+T$ with probability at least $1-\alpha$. Infinite-horizon safety can then be decomposed into (i) safely reaching $\Ccal$ within the finite horizon and (ii) remaining in $\Ccal$ thereafter under the backup policy. Thus, with the switching policy $\pi_S(;s)$ that satisfies~\Cref{assum:exist_backup}, we define the following finite-horizon sufficient condition for~\ref{eq:cc}:
\eqn{ \mathrm{Pr}_{\Pbb_T} \left[H^{x_t,s}_{t}(\w_T) \ge 0  \right] \ge \frac{1 -\varepsilon}{1-\alpha}. \label{eq:fhcc}\tag{FHCC}}
}
\vspace{-15pt}
\subsubsection{Wasserstein Ambiguity Sets}
The distribution $\Pbb_{T}$ is meant to capture the aleatoric uncertainty (inherent stochasticity) in the system and environment at time $t\in \Znonneg$. {In practice, $\Pbb_{T}$ is often unknown.} Thus, we typically rely on a nominal distribution $\hat{\Pbb}_T$ constructed from empirical data or simplified models to evaluate~\ref{eq:fhcc}. This dependence on $\hat{\Pbb}_T$ introduces epistemic uncertainty; a controller designed solely for $\hat{\Pbb}_T$ may fail due to a distributional shift. Consequently, satisfying~\ref{eq:fhcc} under the nominal model does not guarantee safety in the real world.


To achieve robustness to such epistemic uncertainty, we consider an ambiguity set centered around the nominal distribution $\hat{\Pbb}_T$ that contains the true distribution $\Pbb_{T}$:
\begin{definition}[Ambiguity Set]
    Given a nominal noise distribution $\hat{\Pbb}_T \in \Dcal(\Wcal)$, the ambiguity set is $\Bcal_\beta(\hat{\Pbb}_T) = \{ \Qbb ~|~ d(\Qbb,\hat{\Pbb}_T) \le \beta\}$, where $d : \Dcal(\Wcal) \times \Dcal(\Wcal) \to \R_{\ge 0}$ is a distance function between two distributions and $\beta \ge 0$ is the radius.
\end{definition}

One common distance function between distributions is the $\infty$-Wasserstein distance, which measures the smallest worst-case displacement needed to transport mass from one distribution to another:


\begin{definition}[$\infty$-Wasserstein distance~\cite{bertsimas2023data}]
The $\infty$-Wasserstein distance between two probability measures
$\Qbb, \Qbb' \in \mathcal{D}(\mathcal{W})$ is defined as
\eqn{
W_{\infty}(\Qbb,\Qbb')
=
\inf_{\gamma \in \Gamma(\Qbb,\Qbb')}
\gamma\text{-}\operatorname*{ess\,sup}
\|\xi - \xi'\|,}
where $\Gamma(\Qbb,\Qbb')$ denotes the set of all couplings (joint probability measures) with marginals $\Qbb$ and $\Qbb'$, $\|\cdot\|$ is a norm on $\Wcal$, and
\[
\gamma\text{-}\operatorname*{ess\,sup}_{(\xi,\xi') \sim \gamma}
\|\xi - \xi'\|
:=
\inf \left\{
B \ge 0 :
\gamma\big(\|\xi - \xi'\| > B\big) = 0
\right\}.
\]
\end{definition}
In this paper, we will use the $\infty$-Wasserstein distance to define the ambiguity set $\Bcal_\beta(\hat{\Pbb}_T)$.

\begin{assumption}
\label{assum:wasserstein}
    The true distribution $\Pbb_{T}$ lies within an $\infty$-Wasserstein ambiguity set of known radius $\beta \ge 0$ centered at the nominal distribution $\hat{\Pbb}_T$, i.e., $\Pbb_{T} \in \Bcal_\beta(\hat{\Pbb}_T)$.
\end{assumption}

{In practice,  if $\beta$ is unknown, $\beta$ can be treated as a tuning parameter that controls the tradeoff between performance and robustness, which is common in the distributionally robust control literature~\cite{aolaritei2023wasserstein,ham2026dro, hakobyan2023distributionally}. 

We further refine the safety objective in~\ref{eq:fhcc}} using the ambiguity set to formulate a DR chance constraint (DRCC) as a sufficient condition for~\ref{eq:fhcc}, which ensures that the worst-case probability of safety is no less than $\frac{1-\varepsilon}{1-\alpha}$:

    \eqn{{\inf_{\Qbb_T \in \Bcal_\beta(\hat{\Pbb}_T)} \mathrm{Pr}_{\Qbb_T} \left[H^{x_t,s}_{t}(\w_T) \ge 0 \right] \ge \frac{1 -\varepsilon}{1-\alpha}. \label{eq:drcc}\tag{DRCC}}}

A policy $\pi$ is said to satisfy~\ref{eq:drcc} if~\ref{eq:drcc} holds under $\pi$. {Note that satisfaction of~\ref{eq:drcc} implies satisfaction of~\ref{eq:fhcc}. Combined with the probabilistic invariance of $\Ccal$ under~\Cref{assum:exist_backup}, this also implies satisfaction of~\ref{eq:cc}. Hence,~\ref{eq:drcc} provides a distributionally robust sufficient condition for~\ref{eq:cc} without direct knowledge of $\Pbb_T$.}

\subsubsection{Problem Statement}
{
Our objective is to select the largest switching time $s_t^\star$ that satisfies~\ref{eq:cc} at each time $t \in \Znonneg$, resulting in infinite-horizon safety with high probability despite distributional ambiguity.} Instead of directly verifying~\ref{eq:cc}, our approach focuses on~\ref{eq:drcc}. However, verifying~\ref{eq:drcc} for arbitrary nominal distributions $\hat{\Pbb}_T$ is generally intractable~\cite{rahimian2019distributionally}. Hence, we employ a Monte Carlo-inspired sampling approach with finite-sample analysis to select a switching policy that {satisfies~\ref{eq:drcc} with high confidence, in turn satisfying~\ref{eq:cc}}: 
\begin{problem}[Distributionally Robust Safety with Arbitrary Uncertainty Models]
{At each time $t \in \Znonneg$, given a finite planning horizon $T$,} a nominal noise distribution $\hat\Pbb_{T}$ and an ambiguity set $\Bcal_\beta(\hat{\Pbb}_T)$, choose a switching time $s_t^\star$ for the switching policy $\pi_S$ such that 
\eqn{\mathrm{Pr}_N[ \pi_S(;s_t^\star) \textrm{ satisfies \ref{eq:cc}} ]\ge 1-\delta,\label{eq:problem_prob_equality}}
where $\delta \in (0,1)$ is a user-specified error rate {and $\Pr{}_N$ accounts for the randomness from $N \in \Znonneg$ finite samples.}
\label{prob:problem}
\end{problem}

{We note that the above guarantee holds at each time step $t \in \Znonneg$. We will re-compute $s_t^\star$ every time step to update the switching time in a receding horizon manner with the goal of delaying the switch to the backup policy until it is necessary to achieve the desired probability of safety.}

\section{Methodology}

To address~\Cref{prob:problem}, we develop a distributionally robust backup-based safety filter. More specifically, we extend \gk{}~\cite{agrawal2024gatekeeper} which provides an efficient, policy-agnostic safety layer by rolling out a trajectory to determine the latest possible time to switch from a performance-oriented nominal policy to a certified backup safety policy. By prioritizing the nominal policy until the latest possible ``safe'' moment and intervening only when necessary, \gk{} has been shown to be more effective than other backup-based safety filters (cf.~\cite{kim2026backup}).

Despite these benefits, the vanilla \gk{} is restricted to deterministic systems or systems with a known input-to-state stability property w.r.t. a bounded disturbance. Thus, we introduce \sgk{} that retains the efficiency of the original method while providing rigorous probabilistic safety guarantees (see~\Cref{thm:infinite_horizon}). Rather than using a single rollout, \sgk{} rolls out in parallel a set of independent trajectories to ensure probabilistic safety against all disturbance distributions within a Wasserstein ambiguity set containing the unknown true distribution $\Pbb_{{T}}$.

\begin{algorithm}
    \SetKwInOut{Parameters}{Parameters}
    \SetKwInOut{Inputs}{Inputs}
    \caption{\sgk{}}
    \label{alg:gatekeeper}
    \Parameters{$M, T,N, \delta, \varepsilon, \beta, \alpha$}
    \Inputs{$t,x_t, z_t, s_{t-1}^\star$}
    \For{$s \in \{t,\ldots,t+M-1\}$}{
        \For{$i \in \{ 1, \dots, N\}$}{
            \tcp{In parallel}
        Sample $\w_{T,i} \sim {\hat{\Pbb}_T^{x_t,s} = \hat\Pbb^{\pi_S(;s)}_T \otimes \hat\Pbb^\theta_t(\cdot\mid x_t,z_t)}$\;
        Rollout to compute $H^{x_t,s}_{t}(\w_{T,i})$\;
        }

    {$L_H^{x_t,s} = \texttt{GetLipschitz}(x_t, s)$\;}\label{line:lipschitz}
    
    $k^s = \sum_i \mathds{1}(H^{x_t,s}_{t}(\w_{T,i}) {\le} L_H^{x_t,s}\beta)$\; 
    }
     $\rho = 1-(1-\delta)^{1/M}$\;
    $\Ical = \{ s \in \{t,\ldots,t+M-1\}~|~\bar q^s(k^s, N, \rho) \le \frac {\varepsilon-\alpha}{1-\alpha}\}$\;

    \If{$\Ical \ne \emptyset$}
    {
        {$s_t^\star \gets \max\Ical$\;}
    }
    \Else{
    $s_t^\star \gets s_{t-1}^\star$\;
    }
    \Return $s_t^\star$\;
\end{algorithm}

The \sgk{} algorithm is detailed in~\Cref{alg:gatekeeper}, with its flow chart illustrated in~\Cref{fig:flowchart}. At each time $t \in \Znonneg$, the system evaluates $M$ candidate switching times over a rollout horizon $T$, {where $M \le T$.} The algorithm generates $N$ sampled rollouts for each candidate switching time, and uses risk tolerance $\varepsilon$ and ambiguity set radius $\beta$ to provide safety guarantees with confidence level $1-\delta$. Inputs include the current state $x_t$, unsafe set measurements $z_t$, and the previous switching time $s_{t-1}^\star$. We initialize $s_{-1}^\star = 0$.

The procedure follows four main phases. First, the algorithm begins with a parallelized Monte Carlo evaluation of the $M$ candidate switching times. For each candidate switching time $s\in\{t, \dots, t+M-1\}$, we generate $N$ noise trajectory samples {$\w_{T,i} \sim\hat{\Pbb}_T^{x_t,s} = \hat\Pbb^{\pi_S(;s)}_T \otimes \hat\Pbb^\theta_t$ where $\hat{\Pbb}^{x_t,s}_T$ depends on $x_t$, and $s$. For notational simplicity, we subsequently denote this distribution by $\hat{\Pbb}_T$.} Second, the algorithm counts the number of sampled rollouts $k^s$ for which the safety function output falls below the inflated failure buffer $L_H^{x_t,s} \beta$ to account for distributional ambiguity, where $L_H^{x_t,s} \ge 0$ is a Lipschitz constant of $H^{x_t,s}_{t}$ w.r.t. $\w_{T}$, provided by the function $\texttt{GetLipschitz}:\Xcal \times \Znonneg \to \Rnonneg$, and $\beta \ge 0$ is the ambiguity set radius from~\Cref{assum:wasserstein}. Third, $k^s$ is used to evaluate a statistical upper bound $\bar q^s$ on the true failure probability $q^s = \mathrm{Pr}_{\hat\Pbb_T}[H^{x_t,s}_{t}(\w_T) \le L_H^{x_t,s} \beta]$:{
    \begin{equation}
\bar q^s := \bar q^s(k^s, N, \rho) = \max\left\{ q \in [0,1] \;\middle|\; \mathrm{Bin}(k^s; N, q) \ge \rho \right\},
\label{eq:failure_prob_bound}
\end{equation}
such that $\Pr_N[\bar q^s \ge q^s] \ge 1- \rho$ by~\cite[Theorem 4]{vincent2024guarantees} and $\rho = 1 - (1-\delta)^{\frac{1}{M}}$ is the standard {\v{S}}id{\'a}k correction for multi-hypothesis testing~\cite{abdi2007bonferroni}. That is, $\bar q^s$ correctly upper bounds $q^s$ with high confidence. 
 A candidate is valid only if its statistical upper bound $\bar{q}^s$ is less than or equal to the allowable failure probability of $\frac{\varepsilon-\alpha}{1-\alpha}$.} Finally, the algorithm selects the maximum valid switching time within $\Ical$, or reverts to $s_{t-1}^\star$ if no new candidates can be certified.

Once the switching time $s_t^\star$ is reached at some time $t'\geq t$ (i.e., $t'=s_{t'}^\star$), the robot switches to the backup policy $\pi_B$. {If the \ac{PCIS} is then reached, safety is guaranteed with high probability by \Cref{def:backup}}. While using $\pi_B$, the system continues executing~\sgk{}; if a new valid switching time $s_{t''}^\star>t''$ is found at time $t'' > t'$, the robot reverts to $\pi_N$. 


\subsection{Probabilistic Safety Guarantees}
\label{sec:analysis}

{Note that our algorithm samples only from the nominal distribution $\hat{\Pbb}_T$ rather than from the ambiguity set $\Bcal_\beta(\hat{\Pbb}_T)$. However, we certify safety w.r.t. the entire ambiguity set by establishing a theoretical connection between $\hat{\Pbb}_T$ and $\Bcal_\beta(\hat{\Pbb}_T)$ through a Wasserstein-inflated failure threshold (see~\Cref{lem:dist_robust}). This enables us to provide a formal probabilistic safety guarantee in~\Cref{thm:infinite_horizon} that holds regardless of the time horizon $T$ and the number of samples $N$}.

We now rigorously establish that \sgk{} satisfies~\ref{eq:cc} with a desired confidence level of at least $1-\delta$, thereby solving~\Cref{prob:problem}. We first assume the following:

\begin{assumption}
    Let $\Acal := \{\w_{{T}}\in\Wcal ~|~ H^{x,s}_{t}(\w_{{T}}) < 0\}$ and $\Acal^\beta := \{\w_{{T}}\in\Wcal ~|~ d(\w_{{T}},\Acal) {\le} \beta\}$, where $d(\w_{{T}},\Acal):=\inf_{\bar \w_{{T}}\in\Acal}\|\w_{{T}}-\bar \w_{{T}}\|$. $H^{x,s}_{t}(\cdot)$ is Lipschitz continuous on the set $\Acal^\beta$ with Lipschitz constant $L_H^{x,s}$ for all $x \in \Xcal$, $t \in \Znonneg$, and $s \in \{t,\dots,t+M-1\}$. {Furthermore, a Lipschitz constant $L_H^{x,s}$ is available through \texttt{GetLipschitz}}.\label{assum:lip}
\end{assumption}

\begin{remark}
{When the functional forms of  $H^{x_t,s}_{t}$, system dynamics, and their dependence on $\w_T$ are known, a Lipschitz constant $L_H^{x_t,s}$ can be analytically derived. When such an analytical bound is not available, $L_H^{x_t,s}$ can be estimated empirically with data with confidence $1-\delta_{L}$~\cite{knuth2021planning,huang2023sample}.


}
\end{remark}
The following lemma states that the intractable DR worst case probability of failure is upper bounded by the probability of an inflated failure set, which is tractable to compute.
\begin{lemma}
Let Assumptions~\ref{assum:wasserstein} and~\ref{assum:lip} hold. Then,
\eqn{\label{eq:distributional_shift}
\sup_{\Qbb_T \in \Bcal_\beta(\hat\Pbb_T)} &\mathrm{Pr}_{\Qbb_T} [H^{x,s}_{t}(\w_{{T}}) < 0] \nonumber \\ \le ~&\mathrm{Pr}_{\hat\Pbb_T} [H^{x,s}_{t}(\w_{{T}}) \le L_H^{x,s}\beta],}
where $L_H^{x,s}$ is a Lipschitz constant of $H^{x,s}_{t}(\cdot)$ on $\Acal^\beta$.
\label{lem:dist_robust}
\end{lemma}


\begin{proof}
For any $\Qbb_{T}\in\Bcal_\beta(\hat \Pbb_{T})$, there exists a coupling $\gamma\in\Gamma(\Qbb_{T},\hat \Pbb_{T})$ such that $\|\w_{T}-\w_{T}'\|\le \beta$ $\gamma$-almost surely for $(\w_{T},\w_{T}')\sim\gamma$.

Now suppose $\w_{T}\in\Acal$. Since $\|\w_{T}-\w_{T}'\|\le \beta$, we have
\eqnN{d(\w_{T}',\Acal) \le \|\w_{T}'-\w_{T}\| \le \beta.}
Hence $\w_{T}'\in\Acal^\beta$, and therefore
\eqnN{\{\w_{T}\in\Acal\} \subseteq \{\w_{T}'\in\Acal^\beta\} \qquad \gamma\text{-almost surely}.}
Taking probabilities with respect to $\gamma$ gives
\eqnN{\mathrm{Pr}_{\Qbb_{T}}(\Acal) = \gamma(\w_{T}\in\Acal) \le \gamma(\w_{T}'\in\Acal^\beta) = \mathrm{Pr}_{\hat \Pbb_{T}}(\Acal^\beta).}
Since this holds for every $\Qbb_{T}\in\Bcal_\beta(\hat \Pbb_{T})$, we obtain
\eqnN{\sup_{\Qbb_{T}\in\Bcal_\beta(\hat \Pbb_{T})} \mathrm{Pr}_{\Qbb_{T}}(\Acal) \le \mathrm{Pr}_{\hat \Pbb_{T}}(\Acal^\beta).}

It remains to show that $\Acal^\beta \subset \{\w_{T} \in \Wcal ~|~ H_t^{x,s}(\w_{T}) \le L_H^{x,s}\beta\}$. Let $\w_{T}'\in\Acal^\beta$. Then, by definition of $\Acal^\beta$, for any $\eta>0$ there exists $\bar \w_{T}^\eta\in\Acal$ such that $\|\w_{T}'-\bar \w_{T}^\eta\|\le \beta+\eta$. Since $\bar \w_{T}^\eta\in\Acal$, we have $H_t^{x,s}(\bar \w_{T}^\eta)\le 0$. By~\Cref{assum:lip},
\eqnN{H_t^{x,s}(\w_{T}') \le H_t^{x,s}(\bar \w_{T}^\eta) + L_H^{x,s}\|\w_{T}'-\bar \w_{T}^\eta\| \le L_H^{x,s}(\beta+\eta).}
Letting $\eta\downarrow 0$ gives $H_t^{x,s}(\w_{T}') \le L_H^{x,s}\beta$. Hence,
\eqnN{\Acal^\beta {\subseteq} \{\w_{T} \in \Wcal ~|~ H_t^{x,s}(\w_{T}) \le L_H^{x,s}\beta\}.}
Finally, \eqnN{\sup_{\Qbb_{T} \in \Bcal_\beta(\hat \Pbb_{T})} \mathrm{Pr}_{\Qbb_{T}} (\Acal)& \le \mathrm{Pr}_{\hat \Pbb_{T}} (\Acal^{\beta}) \\&\le \mathrm{Pr}_{\hat \Pbb_{T}} (H^{x,s}_T(\w_{T}) \le L_H^{x,s}\beta),}
which completes the proof.
\end{proof}

\begin{theorem}
\label{thm:infinite_horizon}
Let Assumptions \ref{assum:exist_backup}-\ref{assum:lip} hold.
    Let system~\eqref{eq:system} have a nominal policy $\pi_N$ and run~\sgk{} at every time $t \in \Znonneg$ to determine its switching time $s_t^\star$ for the switching policy $\pi_S$~\eqref{eq:switched_time} with $M$ switching times and $N$ rollouts per candidate switching time. Let $\delta,\varepsilon \in(0,1)$ be the desired error and failure rates, and let $\beta \ge 0$ be the radius of the ambiguity set. Then, {if $\Ical\neq \emptyset$ at time $t$, $\pi_S(;s_t^\star)$ satisfies~\ref{eq:cc} with confidence $1-\delta$, i.e.,~\eqref{eq:problem_prob_equality} holds.} 
\end{theorem}

\begin{proof}

For each switching time $s \in \{t,\ldots,t+M-1\}$, let $k^s = \sum_{i=1}^NJ_i^s$ be the observed number of failures, where $J^s_i= \mathds{1}(H^{x_t,s}_{t}(\w_{T,i}) \le L_H^{x_t,s}\beta)$.  Since~\sgk{} estimates the distribution of $H^{x_t,s}_{t}(\w_T)$ by sampling $N$ i.i.d. realizations of the noise as $\w_{T,i}$, there are $N$ Bernoulli random variables $J_i^s$ for each switching time $s$, {from which we compute $\bar q^s$ using~\eqref{eq:failure_prob_bound}.}
    The true (unknown) probability of the rollout failing according to the inflated constraint 
    with the switching time $s$ is $q^s = \mathrm{Pr}_{\hat\Pbb_T}[H^{x_t,s}_{t}(\w_T) \le L_H^{x_t,s} \beta]$.
    
    By~\cite[Theorem 4]{vincent2024guarantees}, computing $\bar{q}^s$ with~\eqref{eq:failure_prob_bound} for each switching time $s\in \{t,\dots,t+M-1\}$ will satisfy $\mathrm{Pr}_N[q^s\leq \bar{q}^s]\geq 1-\rho$, where $\rho = 1 - (1-\delta)^{\frac{1}{M}}$. The set of switching times for which $\bar q^s \le \frac{\varepsilon-\alpha}{1-\alpha}$ is defined as $\Ical = \{s \in \{t,\ldots,t+M-1\} ~| ~ \bar q^s \le\frac{\varepsilon-\alpha}{1-\alpha}\}$.
    {Since the rollouts for all $M$ switching times are independent, we apply the $\text{{\v{S}}id{\'a}k}$ correction for multi-hypothesis testing~\cite{abdi2007bonferroni} which implies $\Pr_N[q^{s_t^\star} \le \bar q^{s_t^\star}] \ge 1- \delta$, where we select $s_t^\star = \max \Ical$.}


 Then, using~\Cref{lem:dist_robust} with~\Cref{assum:wasserstein} and~\ref{assum:lip}, we have
\eqnN{
\sup_{\Qbb_T \in \Bcal_\beta(\hat\Pbb_T)} \mathrm{Pr}_{\Qbb_T}[H_{t}^{x_t,s_t^\star}(\w_T)< 0]
\le q^{s_t^\star} \le \frac{\varepsilon-\alpha}{1-\alpha}.
}
The above condition bounds the failure probability for the worst-case noise distribution. Taking the complement of the failure event, we recover~\ref{eq:drcc}: 
\eqnN{ \inf_{\Qbb_T \in \Bcal_\beta(\hat\Pbb_T)} \mathrm{Pr}_{\Qbb_T}[H_{t}^{x_t,s_t^\star}(\w_T)\ge 0]
\ge 1-\tfrac{\varepsilon-\alpha}{1-\alpha}. }

{Hence, by~\Cref{assum:wasserstein},~\ref{eq:fhcc} holds with confidence $1-\delta$. Then, the system reaches $\Ccal$ at time $t+T$ safely with probability at least $\frac {1-\varepsilon}{1-\alpha}$ and confidence $1-\delta$. Under~\Cref{assum:exist_backup}, the backup policy $\pi_B$ renders $\Ccal$ forward invariant with probability at least $1-\alpha$. Composing these conditions, we have that $s_t^\star$ satisfies~\ref{eq:cc} with confidence at least $1-\delta$.}
\end{proof}

\begin{remark}
\label{remark:per_iteration}
Note that~\Cref{thm:infinite_horizon} provides a confidence guarantee for each invocation of~\sgk{} at time $t$. Whenever~\sgk{} certifies a switching time $s_t^\star$, the corresponding switching policy is guaranteed to satisfy~\ref{eq:cc} with confidence level $1-\delta$. However, repeated invocations over multiple time steps compounds the statistical uncertainty associated with the certification. A probabilistic guarantee over the entire closed-loop horizon can be obtained by reducing $\varepsilon$ and $\delta$ over time via a union bound argument, as in~\cite{bastani2021stat,mestres2025probabilistic}, at the cost of increased conservatism.
\end{remark}


\begin{remark}
\label{remark:backup_for_complex}
With a backup policy $\pi_B$ (\Cref{assum:exist_backup}), \sgk{} ensures~\ref{eq:cc} at state $x_t$ and thus infinite-horizon probabilistic safety. However, constructing such $\pi_B$ offline is often challenging for complex, high-dimensional systems. When no backup policy is available, \sgk{} instead provides a finite-horizon safety guarantee over the horizon $T$. {Studying systematic backup policy construction for such systems is for future work.}
\end{remark}
\section{Simulation}
\begin{table}[t]
\vspace*{6pt}
\scriptsize
\setlength{\tabcolsep}{4pt}
\centering
\caption{Open-Loop Analysis with Varying $\varepsilon$ and $\beta$.}
\resizebox{\linewidth}{!}{%
\begin{tabular}{c|cccc|cccc}
\toprule
 & \multicolumn{4}{c|}{$\varepsilon$} & \multicolumn{4}{c}{$\beta$} \\
 & 0.05 & 0.1 & 0.2 & 0.5 & 0.05 & 0.1 & 0.2 & 0.5 \\
\midrule
Empirical Correctness $[\%]$ & 98 & 99 & 100 & 100 & 94 & 100 & 100 & 100 \\
\bottomrule
\end{tabular}
}
\label{tab:open_loop}
\vspace{-5pt}
\end{table}

We evaluate \sgk{} on three systems: (i) a Dubins vehicle avoiding obstacles, (ii) a Formula 1 car navigating a racing track, and (iii) an F-16 traversing a narrow canyon, assessing safety, performance, and applicability across systems and uncertainty structures. {For (ii) and (iii), constructing and formally verifying a backup policy satisfying~\Cref{assum:exist_backup} is challenging, as discussed in~\Cref{remark:backup_for_complex}. We thus manually design controllers to track states with reduced performance as practical backup policies (e.g., a slow centerline tracker and a climbing controller to escape the canyon).} Evaluations were done on a computer with 12th Gen Intel® Core™ i9-12900KF CPU with 64 GB RAM and an Nvidia RTX 3080 Ti GPU, with parallel rollouts in JAX. Details and videos can be found on our project page.
\footnote{Project Page: \href{https://drs-gk.github.io/}{https://drs-gk.github.io/}}

\begin{table}[t]
\scriptsize
\setlength{\tabcolsep}{4pt}
\centering
\vspace{-2pt}
\caption{Closed-Loop Performance with Varying $\varepsilon$ (left) and $\beta$ (right)}
\resizebox{\linewidth}{!}{%
\begin{tabular}{c|ccc|c|ccc}
\toprule
$\varepsilon$ & \begin{tabular}[c] {@{}c@{}}\textbf{Safe \&}\\ \textbf{Feasible}\\$[\%]$\end{tabular} 
& \begin{tabular}[c]{@{}c@{}}\textbf{Goal}\\\textbf{Reaching}\\\textbf{Time} [s]\end{tabular}
& \begin{tabular}[c]{@{}c@{}}\textbf{Backup}\\\textbf{Ratio}\\ $[\%]$\end{tabular} 
& $\beta$ & \begin{tabular}[c] {@{}c@{}}\textbf{Safe \&}\\ \textbf{Feasible}\\$[\%]$\end{tabular}
& \begin{tabular}[c]{@{}c@{}}\textbf{Goal}\\\textbf{Reaching}\\\textbf{Time} [s]\end{tabular}  
& \begin{tabular}[c]{@{}c@{}}\textbf{Backup}\\\textbf{Ratio} \\ $[\%]$\end{tabular} \\
\midrule
0.01  & 96 & $\infty$    & 100.0 & 0.01   & 92 & 7.52& 21.32 \\
0.05  & 88 & 7.49     & 22.37 & 0.05 & 96 & 7.53 & 19.95 \\
0.1   & 80 & 7.37     & 25.53 & 0.1  & 96 & 7.64 & 20.69\\
0.5   & 12 & 6.50     & 32.23 & 0.5   & 96 & 10.11 & 33.03 \\
\bottomrule
\end{tabular}
}
\label{tab:toy_example}
\end{table}

\begin{table}[t]
\vspace*{6pt}
    \centering
    \caption{Closed-Loop Performance Comparison with Baselines}
    \label{tab:initial_state_comparison}
    \resizebox{\linewidth}{!}{%
\begin{tabular}{l|cccc|c}
\toprule
\textbf{Metric}
& \textbf{DR-CBF}
& \textbf{WT-MPC}
& \textbf{SMPS}
& \textbf{Vanilla GK}
& \textbf{DRS-GK} \\
\midrule
Safety Rate [\%]
& 16
& 28
& 4
& 16
& 96 \\

Mean Time to Goal [s]
& 6.15
& 5.08
& 9.60
& 6.51
& 7.51 \\

Mean Backup Usage [\%]
& --
& --
& 38.02
& 3.29
& 17.06 \\

Mean Computation Time [s]
& 0.0259
& 7.3255
& 0.0103
& 0.0152
& 0.0195 \\
\bottomrule
\end{tabular}
}
\vspace{-15pt}
\end{table}

\subsubsection{\textbf{Dubins Vehicle}}
{
We conduct three sets of experiments with an acceleration-controlled Dubins vehicle tasked with reaching a goal while avoiding obstacles: (i) open-loop to verify~\Cref{thm:infinite_horizon}, (ii) closed-loop 1 to evaluate performance across varying parameter settings, and (iii) closed-loop 2 to compare against baselines.
}

We apply forward Euler discretization with sampling time $\Delta t=0.05$ to the continuous-time uncertain Dubins vehicle model with state $x=\begin{bmatrix} p_x & p_y & \psi & v \end{bmatrix}^\top$:
\eqnN{
\dot{p}_x &= v \cos\psi & \dot{p}_y &= v \sin\psi \\
\dot{\psi} &= (1 + \eta_\omega) u_\omega + v \xi_\omega 
& \dot{v} &= (1 + \eta_a) u_a + v \xi_a
}
The noise vector $\w^f_t= [\xi_a, \xi_\omega, \eta_a, \eta_\omega]^\top$ captures model mismatch and actuator gain errors. {The true noise distribution $\Pbb_T^\pi$ is generated by a finite (truncated) dataset of noise sampled from a Gaussian mixture model. We then perturb each sample by a randomly oriented unit vector scaled by $\beta$ to get $\hat\Pbb_T^\pi$ such that $\Wcal_\infty(\Pbb_T^\pi,\hat\Pbb_T^\pi) \le \beta$.}


The constrained inputs are acceleration $u_a \in [-5.0, 5.0]$ and turn rate $u_\omega \in [-\frac{\pi}{4}, \frac{\pi}{4}]$. Bounds on velocity are enforced in the dynamics such that $v\in[9,20]$, which prevent the vehicle from stopping. The safety constraint is to avoid two static circular obstacles with centers $(p^i_x, p^i_y)$ and radii $R_i$, $i=1,2$. {The obstacle parameters are aggregated as $\theta = [\theta_1^\top, \theta_2^\top]^\top$ where $\theta_i = [p_x^i, p_y^i, R_i]$.} This defines the safe set $\mathcal{S}(\theta) = \{x \in \mathbb{R}^4 \mid h(x,\theta) = \min_{i\in\{1,2\}}\left((p_x - p^i_x)^2 + (p_y - p^i_y)^2 - R_i^2\right) \ge 0\}$.

{We apply the method from~\cite{gao2021computing} to synthesize an \ac{PCIS} for the orbit backup controller, which results in a torus-shaped set in position, accounting for drift from a perfect circular orbital pattern due to noise. We also include noisy measurements of the obstacle centers and radii, which manifest as additive uncertainty in the parameters $\theta$. The known uncertainty $\Pbb^\theta$ is a zero-mean skew normal distribution with $\alpha_{\rm shape} = 5.0$ and $\gamma_{\rm scale} = 1.0$, clipped to $\pm3\gamma_{\rm scale}$.}

\begin{figure}
    \centering
    \includegraphics[width=0.91\linewidth]{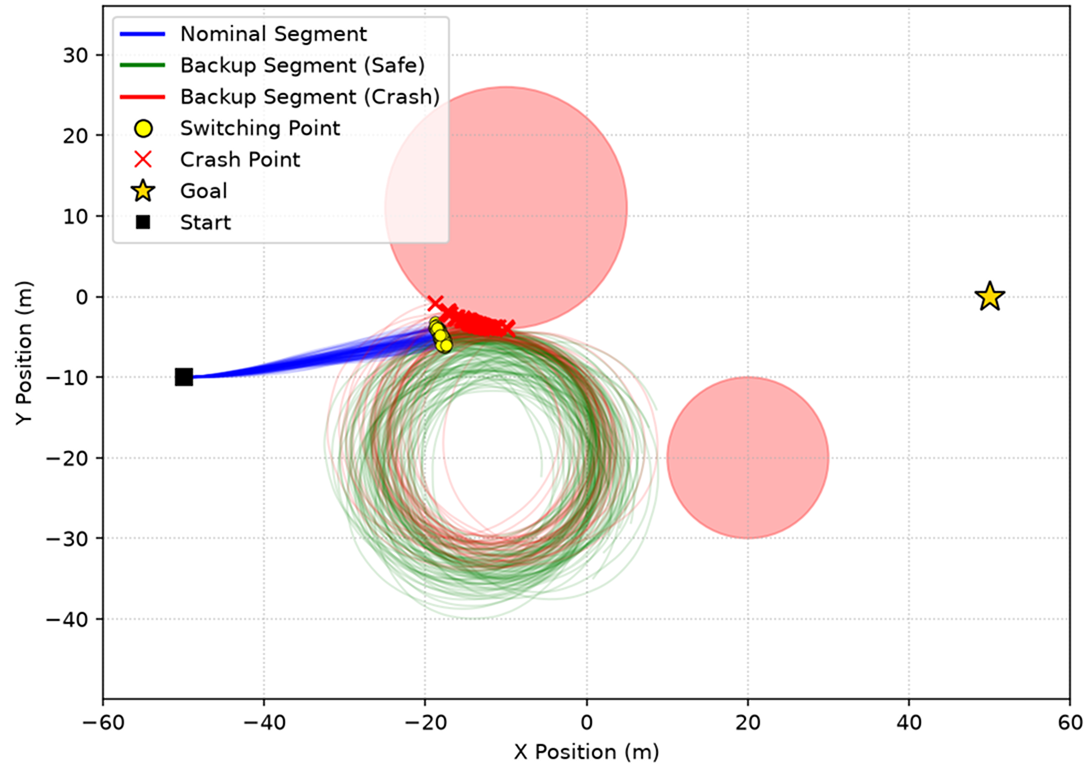}
    \caption{Example of rollouts from a specific initial state and switching time used for open-loop analysis of \sgk{}.}
    \label{fig:open_loop}
\end{figure}

\textbf{Open-Loop:} 
{We empirically validate the $1-\delta$ correctness guarantee of~\Cref{thm:infinite_horizon} using a single iteration of \sgk{} with varying values of $\varepsilon$ and $\beta$. First, we approximate the true failure probability for each switching time $s\in\{0,\ldots,M-1\}$ with $10000$ independent rollouts of $\pi_S(;s)$ over a long horizon $T_{MC}=300$, initialized near an obstacle where many switching times have high failure probabilities. We then compute the empirical confidence as the percentage of the $100$ trials in which the switching time $s_t^\star$ selected by a single \sgk{} iteration satisfies~\ref{eq:cc}. We fix $\delta = 0.1$, $T=50$, $M=50$, and $\alpha=\min\{0.05, \varepsilon/2\}$. \Cref{tab:open_loop} shows that the empirical correctness is consistently above $1-\delta = 0.9$, validating~\Cref{thm:infinite_horizon}. \Cref{fig:open_loop} depicts rollouts with switching time $s=50$ starting at $x=[-50,-10,0,10]$, showing a portion of rollouts than become unsafe after switching to the backup due to close proximity to the obstacle.




}
{\textbf{Closed-Loop 1:}
We also evaluate the effect of the failure probability $\varepsilon$ and radius $\beta$ in \sgk{} by varying them independently over 25 random initial states, while fixing the other at $\varepsilon=0.05$ or $\beta=0.01$. We fix $T = 100$, $M=50$, $\delta=0.01$, $\alpha=\min\{0.05, \varepsilon/2\}$, and $N=1000$. \Cref{tab:toy_example} shows that decreasing $\varepsilon$ or increasing $\beta$ results in more conservative behaviors, improving safety at the cost of higher backup usage and longer goal-reaching times.

\begin{figure}
    \centering
    \includegraphics[width=\linewidth]{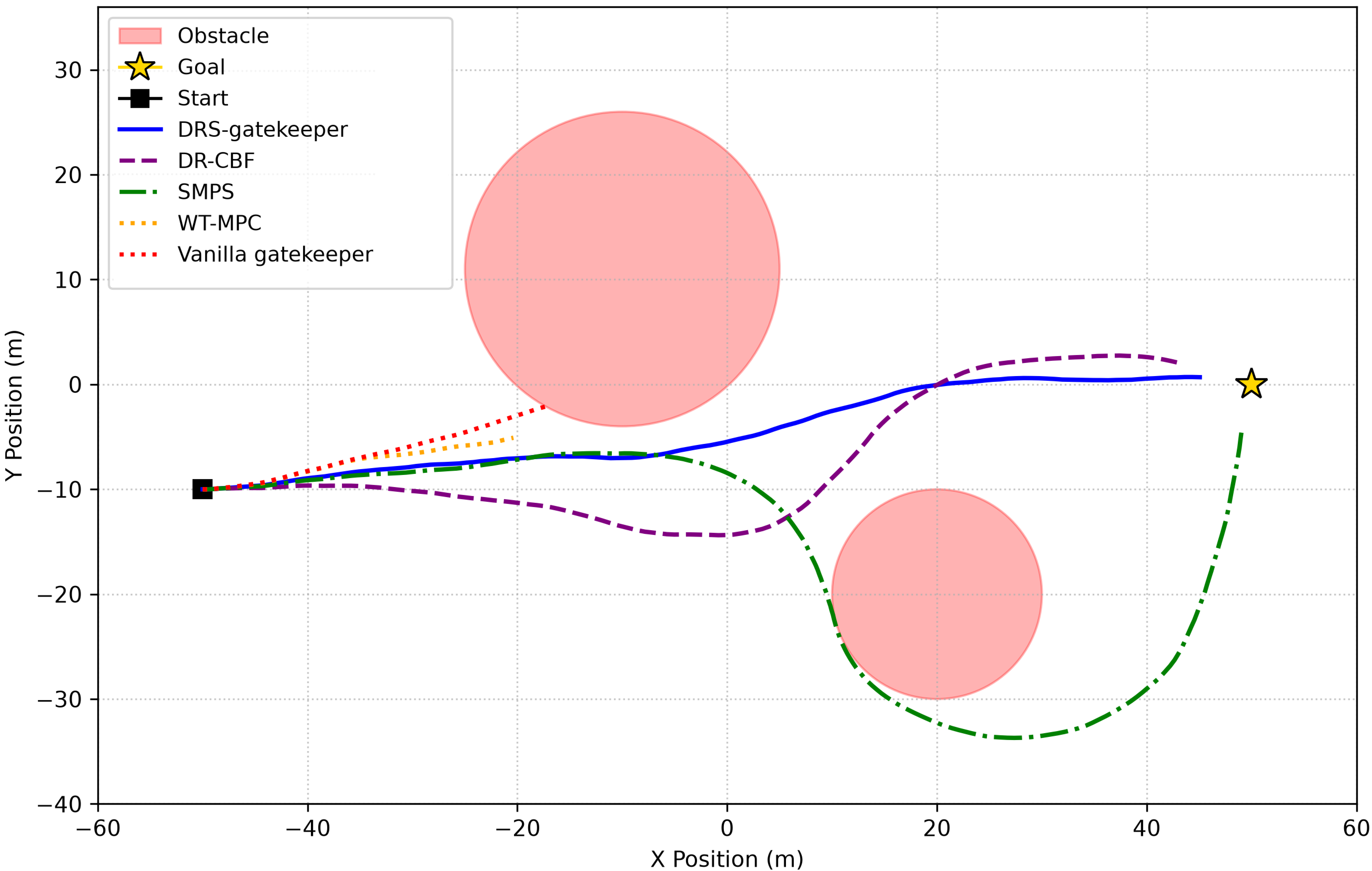}
    \caption{Representative trajectories of \sgk{} and baselines for closed-loop performance comparisons.}
    \label{fig:sim_compare}
\end{figure}

\textbf{Closed-Loop 2:}
We compare \sgk{} to distributionally robust CBF (DR-CBF)~\cite{long2026sensor}, Wasserstein tube MPC (WT-MPC)~\cite{aolaritei2023wasserstein}, statistical model predictive shielding (SMPS)~\cite{bastani2021stat}, and vanilla \gk{}~\cite{agrawal2024gatekeeper}. All methods use a nominal policy that provides a turn rate command proportional to the heading error to the goal and an acceleration command proportional to the target velocity error. The backup-based controllers (SMPS, vanilla~\gk{}, and ours) employ a backup controller that slows the vehicle to its minimum speed and turns away from the nearest obstacle to enter an orbit. WT-MPC uses a linearized Dubins model due to its restriction to linear systems.} 

Representative trajectories of the methods are shown in~\Cref{fig:sim_compare}. WT-MPC and vanilla~\gk{} either become infeasible or result in unsafe trajectories due to unaccounted disturbances. Although SMPS maintains safety, it is more conservative as it only considers triggering the backup controller after one time step. DR-CBF produces the trajectory most comparable to ours; however, due to the solver's myopic behavior, it achieves inferior overall performance. In contrast, our method produces a higher-quality trajectory while maintaining safety.

{We evaluate the methods over 25 random initial states with $\varepsilon=0.1$, $\delta=0.1$, $\beta=0.1$, and $T=100$. For \sgk{}, we fix $M=50$ and $N=1000$. \Cref{tab:initial_state_comparison} reports the empirical safety rate, mean goal-reaching time over successful trials (within 250 time steps), backup usage, and mean computation time, treating solver infeasibility as unsafe for DR-CBF and WT-MPC. \sgk{} achieves the highest safety rate (96\%) while being computationally efficient, at the cost of more frequent backup intervention and a modest increase in goal-reaching time. In contrast, DR-CBF and WT-MPC often become infeasible, while SMPS and vanilla~\gk{} fail to maintain safety under distributional shifts and stochastic disturbances.}

\begin{table}[t]
\vspace{6pt}
\centering
\small
\caption{Performance Summary For Formula 1 Racing Car}
\label{tab:formula_performance}
\small
\setlength{\tabcolsep}{1.5pt}
\renewcommand{\arraystretch}{1.1}

\definecolor{pastelblue}{RGB}{220,235,250}
\definecolor{pastelpink}{rgb}{1.0, 0.82, 0.86}
\definecolor{pastelorange}{RGB}{255,235,220}
\definecolor{pastelgray}{RGB}{245,245,245}
\definecolor{lightmauve}{rgb}{0.86, 0.82, 1.0}

\begin{tabular}{lcccc}
\toprule
\textbf{Method} 
& \begin{tabular}[c]{@{}c@{}}\textbf{Safety} \\ {[\%]}\end{tabular}
& \begin{tabular}[c]{@{}c@{}}\textbf{Avg.}\\\textbf{Speed [m/s]}\end{tabular}
& \begin{tabular}[c]{@{}c@{}}\textbf{Avg. Comp.}\\\textbf{Time [ms]}\end{tabular}
& \begin{tabular}[c]{@{}c@{}}\textbf{Avg. Backup}\\\textbf{ Ratio [\%]}\end{tabular}
\\
\midrule

\rowcolor{pastelgray}
MPPI (no noise)        & 100 & 52.58 & 6.42 & - \\

\rowcolor{pastelpink}
MPPI (Gaussian)        &  92  & 52.17 & 6.35 & - \\

\rowcolor{lightmauve}
MPPI (Empirical)       & 76  & 52.49 & 6.34 & - \\

\midrule
\rowcolor{pastelpink}
Ours (Gaussian) & 94 & 52.05 & 20.90 & $1.04$ \\

\rowcolor{lightmauve}
Ours (Empirical) & 92 & 51.58 & 21.01 & $2.31$ \\
\bottomrule
\end{tabular}
\vspace{-20pt}
\end{table}

\subsubsection{\textbf{Formula 1}}

We consider a Dallara F317 racing car (shown in~\Cref{fig:hero_fig}(a)) with state
$x_t = [p^x_t, p^y_t, \psi_t, v^x_t, v^y_t, r_t, q_t]^\top$,
where $(p^x_t, p^y_t)$ and $\psi_t$ denote global position and heading, $(v^x_t, v^y_t)$ are body-frame velocities, $r_t$ is the yaw rate, and $q_t$ is the steering angle. The control input is $u_t = [ a_{\text{th},t}, a_{\text{br},t},q_{\text{cmd},t}]^\top$, with throttle $0\leq a_{\text{th},t}\leq a_{\max}=12.0$, braking $0\leq a_{\text{br},t}\leq a_{\text{br,max}}=18.0$.\, and steering commands $|q_{\text{cmd},t}|\leq q_{\text{max}}=0.32$. 

We model the car's dynamics using the discrete-time bicycle dynamics $x_{t+1} = f(x_t, u_t)$ as
\eqn{
x_{t+1} = x_t + \Delta t \underbrace{\begin{bmatrix} 
\bar{v}^x_t \cos \psi_t - \bar{v}^y_t \sin \psi_t \\ 
\bar{v}^x_t \sin \psi_t + \bar{v}^y_t \cos \psi_t \\ 
\bar{r}_t \\
a_{\text{long},t} + v^y_t r_t \\
\frac{1}{m}(F_{yf} \cos q_t + F_{yr}) - v^x_t r_t \\
\frac{1}{I_z}(l_f F_{yf} \cos q_t - l_r F_{yr}) \\
k_{\delta}(q_{\text{cmd}, t} - q_{t})
\end{bmatrix}}_{\dot{x}_t}
\label{eq:racecar_dynamics}
}
where $\bar{v}^x_t, \bar{v}^y_t, \bar{r}_t$ are interval-averaged quantities for numerical stability, and $\Delta t=0.05$ is the sampling time. The longitudinal acceleration is
\[
a_{\text{long},t} = a_{\text{th},t} a_{\max} - a_{\text{br},t} a_{\text{br,max}} - \frac{C_d}{m} v^x_t |v^x_t|.
\]
Lateral tire forces follow a linear model
\[
F_{y,\{f,r\}} = C_{\{f,r\}} \alpha_{\{f,r\}},
\]
with slip angles $\alpha_f = q_t - \arctan\!\big(\frac{v^y_t + l_f r_t}{v^x_t}\big)$ and 
$\alpha_r = -\arctan\!\big(\frac{v^y_t - l_r r_t}{v^x_t}\big)$.

To identify model parameters, we use a dataset of offline laps on the $4.657 \ \mathrm{km}$ Barcelona-Catalunya GP circuit obtained from~\cite{remonda2024simulation}, where the data were collected using Assetto Corsa. The empirical model uncertainty (visualized in~\Cref{fig:car_uncertainty}) is quantified via one-step prediction residuals in the longitudinal/lateral velocities and yaw rate between the nominal dynamics model and the ground-truth telemetry.
\eqnN{
\delta v^x_t &= v^x_{t+1,\text{gt}} - (v^x_t + \dot{v}^x_t \Delta t), \\
\delta v^y_t &= v^y_{t+1,\text{gt}} - (v^y_t + \dot{v}^y_t \Delta t), \\
\delta r_t &= r_{t+1,\text{gt}} - (r_t + \dot{r}_t \Delta t).
}

\begin{figure}
    \centering
    \includegraphics[width=0.8\linewidth]{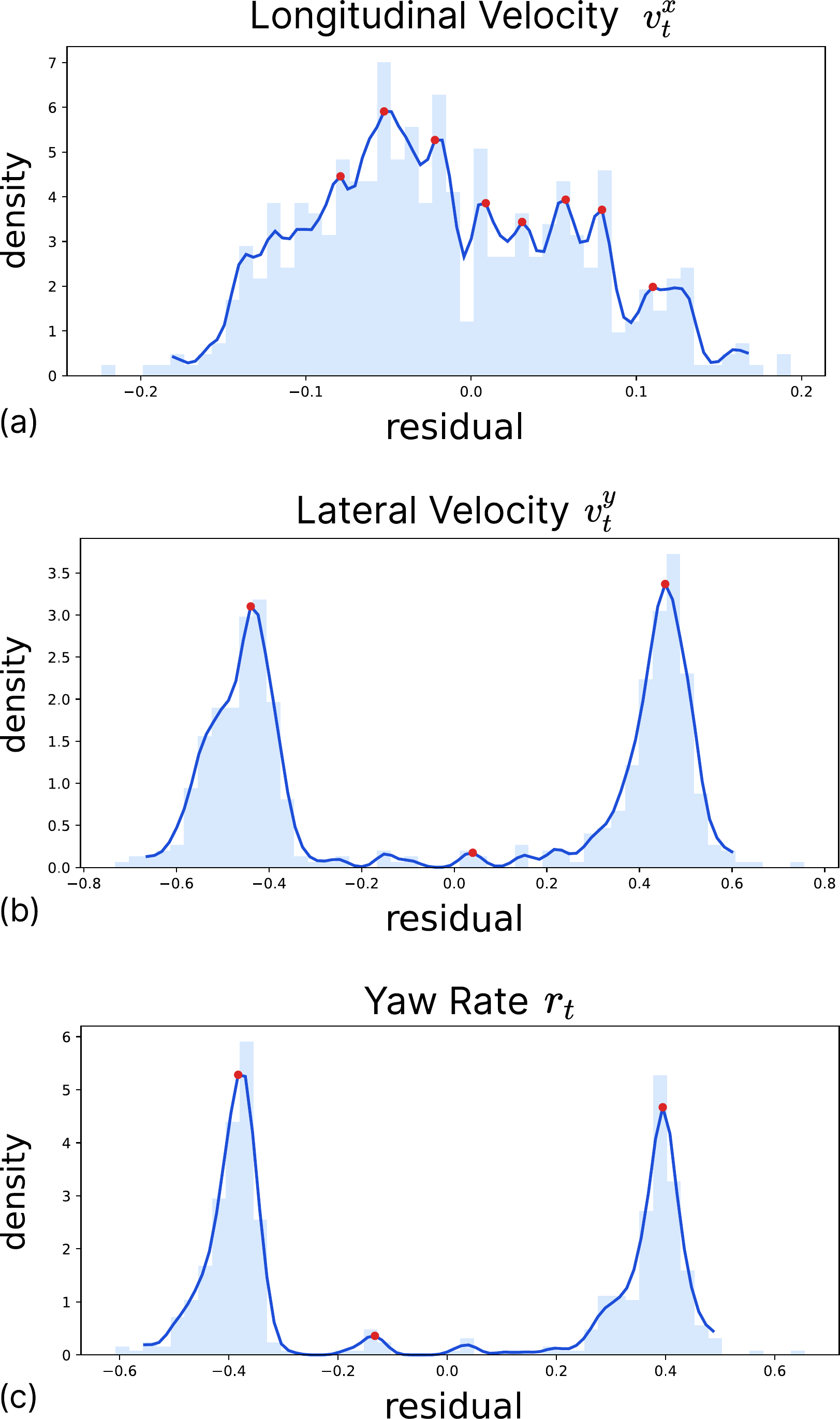}
    \caption{Empirical uncertainty distributions for (a) longitudinal velocity, (b) lateral velocity, and (c) yaw rate.}
    \label{fig:car_uncertainty}
\end{figure}

The safety function measures distance to the track boundary (width $6.0 \ \mathrm{m}$). We use Model Predictive Path Integral (MPPI)~\cite{williams2018information} as the nominal policy $\pi_N$, tuned to track the centerline at $55 \ \mathrm{m/s}$ without leaving the track. The backup policy $\pi_B$ is a PID controller that keeps the vehicle within $2.75 \ \mathrm{m}$ of the centerline while maintaining $6 \ \mathrm{m/s}$. We use $M=10$, $T=10$, $N=1000$, $\delta=0.1$, and $\varepsilon=0.1$. Based on 100 offline laps, we set $L_H^{x,s}=2.0$ for all $x$ and $s$.

We conduct an ablation study under two noise settings:
\begin{itemize}
    \item Gaussian: The true distribution $\Pbb_T$ is i.i.d. zero-mean Gaussian with $\sigma=(0.45,0.45,0.25,0.015)$ for longitudinal/lateral velocities, yaw rate, and steering angle, respectively, clipped to $\pm4\sigma$.
    
    \item Empirical: ${\Pbb}_T$ is state-dependent and conditioned on the past five control inputs. It reflects uncertainty in longitudinal/lateral velocities and yaw rate. Perception uncertainty in the track boundary is modeled as zero-mean Gaussian ($\sigma=0.50$) clipped to $\pm4\sigma$.
\end{itemize}

For the Gaussian setting, we set $\hat{\Pbb}_T=\Pbb_T$ and sample directly from $\Pbb_T$ with $\beta=0$. For the empirical setting, we construct $\hat{\Pbb}_T$ with two sources of model mismatch: (i) the safety function $H^{x_t,s}_{t}$ is evaluated without accounting for the injected Gaussian measurement noise, and (ii) the state uncertainty noise is sampled conditioned only on the current control input rather than the full input history. We then construct the ambiguity set $\Bcal_\beta(\hat{\Pbb}_T)$ with radius $\beta=0.25$.

We run 50 trials on the Barcelona-Catalunya GP circuit with and without~\sgk{}, each attempting to complete a half lap with varying seeds. \Cref{tab:formula_performance} reports safety (percentage of successful half laps), average speed, computation time, and backup ratio. Computation time for our method includes both MPPI and~\Cref{alg:gatekeeper}.

Our approach consistently improves safety across all noise settings. While MPPI alone performs well under Gaussian setting, it degrades significantly under empirical setting with distributional shift. In contrast, our method achieves over $90\%$ success under both settings, with a slight reduction in speed, minimal backup intervention ($<3\%$), and a modest computational overhead of only $14$-$15~\mathrm{ms}$.











\begin{table}[t]
\vspace{6pt}
\centering
\small
\caption{Performance Summary For F-16 Fighter Jet}
\begin{tabular}{l c c c}
\toprule
\textbf{Metric} & \textbf{Baseline}  & \textbf{Ours}\\
\midrule
\textbf{Safety [\%]} & $40$ & $88$&\\
\textbf{Avg. Altitude Below Canyon [m]} & $-639$ & $-618$ & \\
\textbf{Avg. Speed [m/s] }& $200$ & $194$&\\
\textbf{Avg. Computation Time [ms]} & $0.1$ & $12.2$ &\\
\textbf{Avg. Backup Usage [\%]} & - & $13.9$ \\
\bottomrule
\end{tabular}
\label{tab:f16_performance}
\vspace{-20pt}
\end{table}
\subsubsection{\textbf{F-16}}

Next, we evaluated our method on an F-16 fighter jet flying as low and as fast as possible through a narrow canyon in JSBSim~\cite{berndt2004jsbsim}, as visualized in~\Cref{fig:hero_fig}~(b). A 4.7-km-long optimal trajectory was generated offline using a simplified three-degrees-of-freedom (3DoF) model using the method in~\cite{sharpe2025accelerating} and tracked online using proportional-derivative (PD) control to achieve angle of attack, bank angle, and speed commands on the full 12-state and 4-control-input F-16 model. We collected trajectories in JSBSim with turbulence and fit a third-degree polynomial model of the six aerodynamic coefficients, with uncertainty modeled by the residuals between the data and polynomial fits.

The F-16 aircraft has state vector $x_t = [p^N_t, p^E_t, h_t, u_t, v_t, w_t, p_t, q_t, r_t, \phi_t, \theta_t, \psi_t]^\top$, where $(p^N_t, p^E_t, h_t)$ and $(\phi_t,\theta_t,\psi_t)$ denote inertial position and Euler attitude, $(u_t,v_t,w_t)$ are body-frame velocities, and $(p_t,q_t,r_t)$ are body angular rates. The control input is $u_t = [\delta_{A,t}, \delta_{E,t}, \delta_{R,t}, \delta_{T,t}]^\top$, with aileron, elevator, rudder, and throttle commands\footnote{Note that $\delta$, $\alpha$, and $\beta$ in the F-16 dynamics denote control inputs, angle of attack, and sideslip angle, and are distinct from \sgk{} parameters.}.

Let $c_{(\cdot)}=\cos(\cdot)$, $s_{(\cdot)}=\sin(\cdot)$, and $t_{(\cdot)}=\tan(\cdot)$. We model the aircraft dynamics using a discrete-time rigid-body model $x_{t+1} = f(x_t, u_t)$ as
\eqn{\label{eq:fighterjet_dynamics}
x_{t+1} = x_t + \Delta t \underbrace{\begin{bmatrix}
\dot p^N_t \\
\dot p^E_t \\
\dot h_t \\
X_t + r_t v_t - q_t w_t - g s_{\theta_t} \\
Y_t + p_t w_t - r_t u_t + g s_{\phi_t} c_{\theta_t} \\
Z_t + q_t u_t - p_t v_t + g c_{\phi_t} c_{\theta_t} \\
\dot p_t \\
\dot q_t \\
\dot r_t \\
p_t + q_t s_{\phi_t} t_{\theta_t} + r_t c_{\phi_t} t_{\theta_t} \\
q_t c_{\phi_t} - r_t s_{\phi_t} \\
(q_t s_{\phi_t} + r_t c_{\phi_t})/c_{\theta_t}
\end{bmatrix}}_{\dot x_t}
}
where $\Delta t=\frac{1}{30}$ is the sampling time and $g=32.174\,\mathrm{ft/s^2}$. The position kinematics are
\[
\begin{aligned}
\dot p^N_t &= u_t(c_{\theta_t}c_{\psi_t}) + v_t(s_{\phi_t}s_{\theta_t}c_{\psi_t} - c_{\phi_t}s_{\psi_t}) \\
&\quad + w_t(c_{\phi_t}s_{\theta_t}c_{\psi_t} + s_{\phi_t}s_{\psi_t}), \\
\dot p^E_t &= u_t(c_{\theta_t}s_{\psi_t}) + v_t(s_{\phi_t}s_{\theta_t}s_{\psi_t} + c_{\phi_t}c_{\psi_t}) \\
&\quad + w_t(c_{\phi_t}s_{\theta_t}s_{\psi_t} - s_{\phi_t}c_{\psi_t}), \\
\dot h_t &= u_t s_{\theta_t} - v_t s_{\phi_t}c_{\theta_t} - w_t c_{\phi_t}c_{\theta_t}.
\end{aligned}
\]

Aerodynamic coefficients are predicted via polynomial approximation
\[
[\hat C_{X,t},\hat C_{Y,t},\hat C_{Z,t},\hat C_{L,t},\hat C_{M,t},\hat C_{N,t}]
= \phi(\upsilon_t)^\top W + B,
\]
where $\upsilon_t=[\alpha_t,\beta_t,\mathrm{Mach}_t,p_t,q_t,r_t,\delta_{e,t},\delta_{a,t},\delta_{r,t}]^\top$, and
$\alpha_t=\arctan2(w_t,u_t), ~ \beta_t=\arcsin\!\left(\frac{v_t}{V_t}\right), ~ V_t=\sqrt{u_t^2+v_t^2+w_t^2}.$
Using dynamic pressure $q_{\infty,t}=\frac12\rho V_t^2$, wing area $S$, span $b$, mean chord $\bar c$, and mass $m$:
\[
\begin{aligned}
X_t &= q_{\infty,t}\frac{S}{m}\hat C_{X,t}+T(\delta_{t,t}), & Y_t &= q_{\infty,t}\frac{S}{m}\hat C_{Y,t}, \\
Z_t &= q_{\infty,t}\frac{S}{m}\hat C_{Z,t}, & L_t &= q_{\infty,t}Sb\,\hat C_{L,t}, \\
M_t &= q_{\infty,t}S\bar c\,\hat C_{M,t}, & N_t &= q_{\infty,t}Sb\,\hat C_{N,t},
\end{aligned}
\]
and angular accelerations satisfy
\[
\dot\omega_t = I^{-1}\!\left(\begin{bmatrix}L_t\\M_t\\N_t\end{bmatrix} - \omega_t \times (I\omega_t)\right),\qquad
\omega_t=[p_t,q_t,r_t]^\top.
\]
Control inputs are normalized and bounded by $\delta_{a,t},\delta_{e,t},\delta_{r,t}\in[-1,1]$, $\delta_{t,t}\in[0,1]$.

\begin{figure}[h]
\centering
\includegraphics[width=\linewidth]{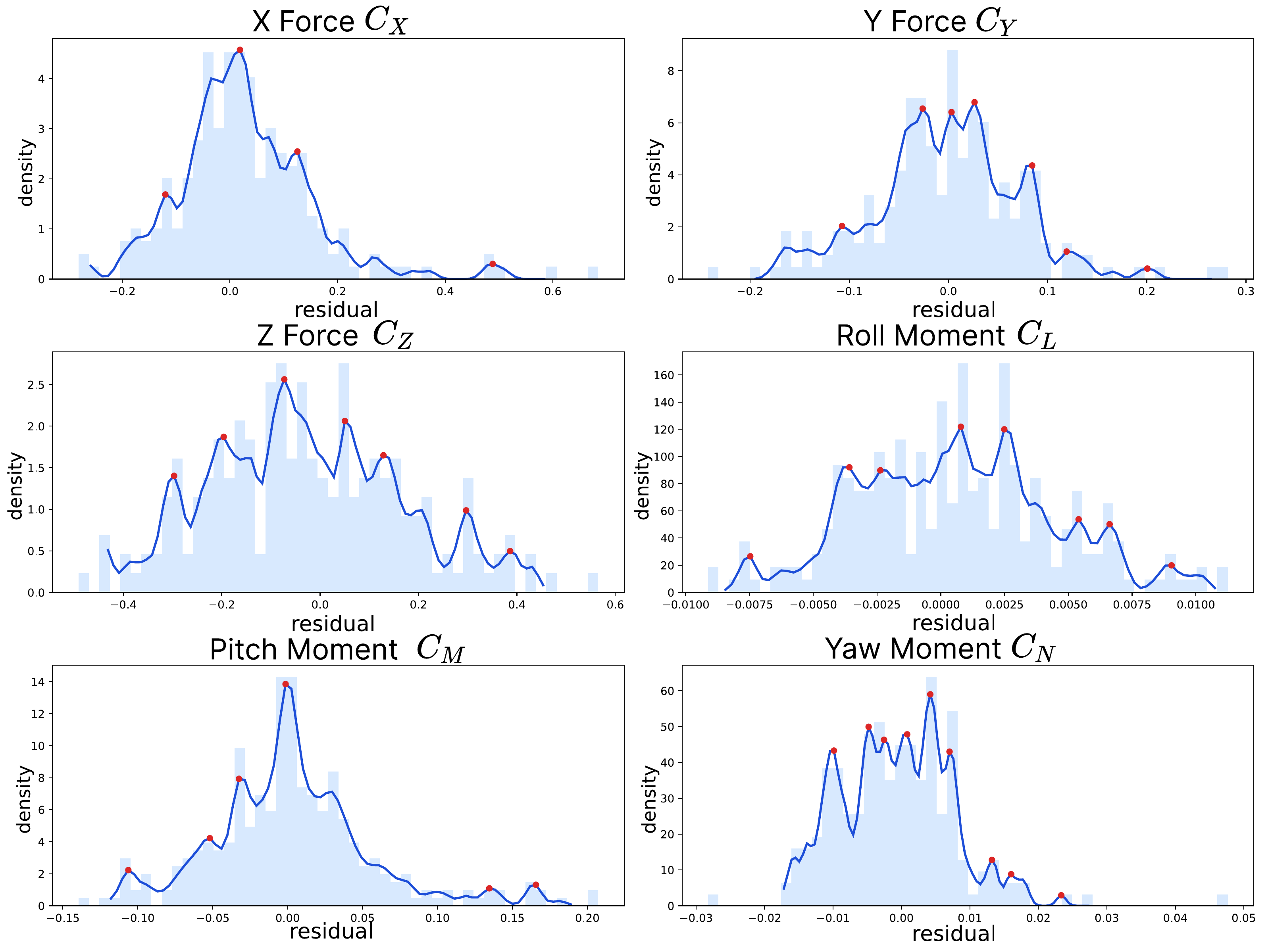}
\caption{F-16 empirical uncertainty distributions conditioned on a specific flight condition.}
\end{figure}

The nominal aerodynamic model is identified from generated simulator data. Each rollout provides state, action, and next-state samples, converted into aerodynamic coefficient targets. A regularized polynomial regression is fit from the selected flight and control features to the six aerodynamic coefficient channels:
\[
\mathbf{C}_{\mathrm{nom}} = [C_X, C_Y, C_Z, C_L, C_M, C_N]^{\top}.
\]

After fitting the nominal polynomial model, residual coefficient errors are computed on the same generated dataset. Structured residual trends are removed using context-dependent calibration features, namely flight condition. Centered residuals are stored as an empirical uncertainty model. During rollout, uncertainty samples are drawn from residuals observed in nearby historical contexts; injected uncertainty remains dependent on the current flight regime rather than being modeled as fixed Gaussian noise.

In the \sgk{} rollout, sampled uncertainty enters directly in coefficient space. Let
\[
\mathbf{w} = [w_{C_X}, w_{C_Y}, w_{C_Z}, w_{C_L}, w_{C_M}, w_{C_N}]^{\top}
\]
denote the sampled coefficient residual. Perturbed aerodynamic coefficients are $\mathbf{C} = \mathbf{C}_{\mathrm{nom}} + \mathbf{w}$. Dimensional forces and moments are converted via:
\[
\begin{aligned}
X &= (C_X + w_{C_X})\frac{\bar q S}{m} + T(\delta_t), & Y &= (C_Y + w_{C_Y})\frac{\bar q S}{m}, \\
Z &= (C_Z + w_{C_Z})\frac{\bar q S}{m}, & L &= (C_L + w_{C_L})\bar q S b, \\
M &= (C_M + w_{C_M})\bar q S \bar c, & N &= (C_N + w_{C_N})\bar q S b.
\end{aligned}
\]

Uncertainty is injected before rigid-body propagation by perturbing aerodynamic coefficients, changing body-axis force and moment channels used by the rollout dynamics.

The safety function is defined as the distance to the canyon wall. The mission is to fly through the canyon as low and as fast as possible. The backup policy is implemented as PD-controller that tracks an altitude of $1500$ ft, slightly above the canyon walls, with a backup speed target of $300$ kts. We set parameters to $M=25$, $T=100$, $N=256$, $\delta=0.1$, $\varepsilon=0.05$, and $\beta=1.0$. We also fixed $L_H^{x,s}$ to $50.0$ for all $x$ and $s$. The controllers run at $30$ Hz and \sgk{} runs at $3$ Hz, or every 10 steps. During the intermediate steps, the most recently computed switching time is used. 

We simulate $50$ trials with and without \sgk{}, where uncertainties arise from turbulence and residual errors in the polynomial dynamics fit. \Cref{tab:f16_performance} shows that our method significantly improves safety compared to just executing the tracking controller, as the offline optimal trajectory is generated using a simplified dynamics model without wind. This improvement is achieved with minimal backup intervention, only a slight reduction in speed and increase in altitude, while maintaining low computational cost.

\section{Conclusion}
\label{sec:conc}
We introduce Distributionally Robust Stochastic \gk{} (\sgk{}), a safety filter that provides probabilistic safety certificates under distributional ambiguity. Our method handles arbitrary uncertainty structure and nonlinear dynamics through efficient sampling-based rollouts and optimizes a switching time between nominal and backup policies. We conduct extensive simulations to validate our method.



\bibliographystyle{IEEEtran}
\bibliography{references_ll}

\end{document}